\documentclass[conference]{IEEEtran}
\IEEEoverridecommandlockouts
\usepackage{cite}
\usepackage{url}
\usepackage{amsmath,amssymb,amsfonts}
\usepackage{algorithmic}
\usepackage{graphicx}
\usepackage{textcomp}
\usepackage{xcolor}
\usepackage{amsthm}
\def\BibTeX{{\rm B\kern-.05em{\sc i\kern-.025em b}\kern-.08em
    T\kern-.1667em\lower.7ex\hbox{E}\kern-.125emX}}
\begin{document}

\title{Extensions to Generalized Annotated Logic and an Equivalent Neural Architecture
}

\author{\IEEEauthorblockN{Paulo Shakarian}
\IEEEauthorblockA{\textit{Arizona State University}\\
Tempe, AZ USA \\
pshak02@asu.edu}
\and
\IEEEauthorblockN{Gerardo I. Simari}
\IEEEauthorblockA{\textit{Department of Computer Science and Engineering}\\
\textit{Universidad Nacional del Sur (UNS)}\\
\textit{Institute for Comp.\ Sci.\ and Eng.\ (UNS-CONICET)}\\
Bahia Blanca, Argentina \\
gis@cs.uns.edu.ar}
}

\newtheorem{dmodel}{Diffusion Model}[section]
\newtheorem{definition}{Definition}[section]
\newtheorem{example}{Example}[section]
\newtheorem{proposition}{Proposition}[section]
\newtheorem{theorem}{Theorem}[section]
\newtheorem{lemma}[theorem]{Lemma}
\newtheorem{fact}[theorem]{Observation}

\newtheorem{oprob}{Open Problem}[section]
\newtheorem{mex}{Military Example}[section]
\newtheorem{corollary}{Corollary}[section]

\newcommand{\from}[3]{{\bf [{\sc from #1 to #2:} {\small #3}]}}

\def\asetpi{\textsf{aset}_{\Pi(GK)}}
\def\xa{X_{\bV_1}}
\def\xc{X_{\bV_2 - \bV_1}}
\def\xb{X_{\{ v \}}}
\def\cmin{c_{(\min)}}
\def\minsetpi{\textsf{minset}_{\Pi(GK)}}
\def\maxsetpi{\textsf{maxset}_{\Pi(GK)}}
\def\upmaxpi{\textsf{UPMAX}_{\Pi(GK)}}
\def\lomaxpi{\textsf{LOMAX}_{\Pi(GK)}}
\def\minsetpio{\textsf{minset}_{\Pi^{1}(GK)}}
\def\maxsetpio{\textsf{maxset}_{\Pi^{1}(GK)}}
\def\upmaxpio{\textsf{UPMAX}_{\Pi^{1}(GK)}}
\def\lomaxpio{\textsf{LOMAX}_{\Pi^{1}(GK)}}

\def\pio{\Pi^{(1)}}
\def\dppi{DP_{\Pi(GK)}}
\def\dppio{DP_{\pio(GK)}}
\def\vdp{V_{DP}}
\def\edp{E_{DP}}
\def\lbldp{lbl_{DP}}

\def\nset{\textsf{nset}_{GK,pred}}
\def\nsetgoal{\textsf{nset}_{GK,goal}}

\def\asetpio{\textsf{aset}_{\pio(GK)}}

\def\covset{\textsf{covset}_{\pio(GK)}}

\def\md{\textsf{mindist}}
\def\fr{\textsf{fr}}

\def\relu{\mathit{relu}}
\def\incon{\mathit{incon}}
\def\percep{\mathit{percep}}

\def\calr{\mathcal{R}}
\def\calb{\mathcal{B}}
\def\calc{\mathcal{C}}
\def\cala{\mathcal{A}}
\def\calk{\mathcal{K}}
\def\cali{\mathcal{I}}
\def\calt{\mathcal{T}}
\def\calv{\mathcal{V}}

\def\call{\mathcal{L}}

\def\call{\mathcal{L}}
\def\calo{\mathcal{O}}
\def\cald{\mathcal{D}}
\def\cals{\mathcal{S}}
\def\cale{\mathcal{E}}
\def\calf{\mathcal{F}}
\def\calu{\mathcal{U}}
\def\feas{{\mathsf{feas}}}
\def\eqo{\textsf{EQ}_\calo}
\def\eqs{\textsf{EQ}_\cals}
\def\ordo{\textsf{ORD}_\calo}
\def\ords{\textsf{ORD}_\cals}
\def\GK{\textsf{GK}}
\def\FM{{\mathsf{FM}}}
\def\preans{{\mathsf{pre\_ans}}}
\def\answ{{\mathsf{ans}}}
\def\calh{{\mathcal{H}}}

\def\met{\textsf{MET}}

\def\pred{\mathit{pred}}
\def\npred{\mathit{not\_pred}}
\def\cpred{\mathit{cause\_pred}}

\def\rel{\mathit{rel}}
\def\red{\textsf{RED}}
\def\bl{\textsf{BLUE}}
\def\none{\textsf{none}}
\def\rsc{\textsf{RSC}}
\def\pd{\textsf{pd}}
\def\sp{\mathcal{S}}
\def\APT{{\textsf{APT}}}
\def\ISA{{\textsf{ISA}}}
\def\apt{\APT}
\def\calp{\mathcal{P}}
\def\st{\textit{ such that }}
\def\pst{P^{*}}
\def\pstp{P^{*}_{(Pr)}}
\def\pms{\textsf{PM}_\sp}
\def\pc{\textsf{PC}}
\def\mpc{\textsf{mPC}}
\def\setcov{\textsf{SET\_COVER}}
\def\ssa{\textsf{SIM\_SA}}
\def\msa{\textsf{STEADY\_SA}}
\def\NAP{{\textsf{NAP}}}
\def\NAPs{{\textsf{NAPs}}}
\def\avar{\textsf{AVar}}
\newcommand{\tib}[1]{{\textbf{\textit{#1}}}}
\def\T{{\mathbf{T}}}
\def\vp{{\textsf{VP}}}
\def\ep{{\textsf{EP}}}
\def\bG{\textbf{\textsf{G}}}
\def\bV{\textbf{\textsf{V}}}
\def\bE{\textbf{\textsf{E}}}

\def\val{\textit{val}}
\def\inc{\textit{inc}}
\def\incalg{\textit{inc}^{(\textit{alg})}}
\def\incopt{\textit{inc}^{(\textit{opt})}}
\def\incup{\textit{inc}^{(\textit{up})}}
\def\incoptup{\textit{inc}^{(\textit{opt,up})}}
\def\Ialg{I^{(\textit{alg})}}
\def\remaining{\textit{REM}}
\def\INC{\textit{INC}}
\def\PROG{\textit{PROG}}
\def\Ie{I^{(\epsilon)}}
\def\INCe{\textit{INC}^{(\epsilon)}}
\def\ince{\textit{inc}^{(\epsilon)}}
\def\PROGe{\textit{PROG}^{(\epsilon)}}
\def\spreade{\textit{spread}^{(\epsilon)}}
\def\Gspreade{\textit{GS}^{(\epsilon)}}
\def\cande{{\textit{cand}^{(\epsilon)}}}
\def\S{\textbf{S}}

\maketitle

\begin{abstract}
While deep neural networks have led to major advances in image recognition, language translation, data mining, and game playing, there are well-known limits to the paradigm such as lack of explainability, difficulty of incorporating prior knowledge, and modularity.  Neuro symbolic hybrid systems have recently emerged as a straightforward way to extend deep neural networks by incorporating ideas from symbolic reasoning such as computational logic.
In this paper, we propose a list desirable criteria for neuro symbolic systems and examine how some of the existing approaches address these criteria.  We then propose an extension to generalized annotated logic that allows for the creation of an equivalent neural architecture comprising an alternate neuro symbolic hybrid.  However, unlike previous approaches that rely on continuous optimization for the training process, our framework is designed as a binarized neural network that uses discrete optimization.
We provide proofs of correctness and discuss several of the challenges that must be overcome to realize this framework in an implemented system.
\end{abstract}


\section{Introduction}

While deep neural networks have led to major advances in image recognition, language translation, data mining, and game playing, there are well-known limits to the paradigm such as lack of explainability, difficulty of incorporating prior knowledge, and modularity~\cite{marcus18}.  Neuro symbolic hybrid systems have recently emerged as a straightforward way to extend deep neural networks by incorporating ideas from symbolic reasoning such as computational logic~\cite{neurASP2020,thomasNsr21,nttIlp21,deepMinIlp2018,lnn2020,ltn22,fuzzy2022}.
In this paper, we first propose a list of desirable criteria for a neuro symbolic system and examine how some of the existing approaches each address these criteria.
We then propose an extension to generalized annotated logic~\cite{ks92} that allows for the creation of an equivalent neural architecture comprising an alternate neuro symbolic hybrid.
However, unlike previous approaches that rely on continuous optimization for the training process, our framework is designed as a binarized neural network~\cite{Courbariaux16} that uses discrete optimization.  We provide proofs of correctness and discuss several of the challenges that must be overcome to realize this framework in an implemented system.

The rest of this paper is organized as follows.
In Section~\ref{sec:cri} we propose a list of desirable criteria and briefly discuss which of the current approaches meet the criteria.  In Section~\ref{sec:gaps} we review generalized annotated logic along with our extensions, mainly related to the use of interpretations for literals (as opposed to atoms) and the use of a lower semi-lattice (instead of an upper semi-lattice).
This is followed by a section containing proofs related to our extensions to annotated logic (Section~\ref{sec:lowerRes}).
This allows us to then introduce our parametrized rule structure (Section~\ref{sec:struct}) and associated neural architecture (Section~\ref{sec:neuralEmbed}).
We discuss how this framework can identify inconsistencies in Section~\ref{sec:incon}, briefly discuss in Section~\ref{sec:discusion} how our framework meets the criteria,
and conclude with a discussion of challenges (Section~\ref{sec:chall}).

\section{Criteria to Consider in Neuro Symbolic Reasoning and Related Work}
\label{sec:cri}

Inspired on the critique of deep learning in~\cite{marcus18}, we have derived the following criteria for a neuro symbolic system.
\begin{enumerate}
\item Support for symbolic inference for arbitrary queries
\item Symbolic explanation of results amenable to both further analysis and human interpretation
\item Ability to integrate prior knowledge and/or constraints
\item Strong assurances of consistency
\item Ability to learn rule structure, including classical (i.e., non-fuzzy) rules
\item Scalability
\end{enumerate}

The framework of logic tensor networks (LTN)~\cite{ltn22,fuzzy2022} provides a neural symbolic approach where symbols have an associated vector representation that in turn allows, via the training process, for logical sentences (known a-priori) to be assigned weights where the intuition is related to a level of truth.  This follows from fuzzy logic, similar to prior approaches combining fuzzy logic and machine learning~\cite{getoor2017}.  The framework excels in its ability to conduct symbolic inference on queries and provide symbolic results in a scalable manner as well as the ability to incorporate prior knowledge.  However, the framework does not allow for the learning of logical rules and it treats consistency as a component of the loss function, hence not guaranteeing consistency.  We note that the lack of consistency assurances means that we do not have guarantees that prior knowledge is indeed integrated.  Further, the system does not provide an explanation of how it obtained a symbolic result.  In many ways, the capabilities of LTN's are similar to deep ontological networks~\cite{hohenecker2020ontology} (although the two differ significantly in how they handle training data).  However,  the explainability problem is more pronounced in ontological neural networks due to its use of embeddings.  In summary, LTNs and deep ontological networks meet criteria 1, 3, and 6---they allow for symbolic reasoning and incorporation of prior knowledge in a scalable manner but do not make consistency guarantees, learn rule structure, nor provide explanatory results.

The framework of logical neural networks (LNN)~\cite{lnn2020} overcomes some of the explainability issues of LTN and deep ontological networks.  However, like LTN, it still does not have hard consistency guarantees, as consistency is handled as a component of the loss function.  Further, it is not capable of learning rules without a-prior knowledge of their structure.  In short, LNNs meet criteria 1, 2, 3, and 6---they provide much of the same capabilities of LTN and deep ontological networks, but also provide an explainable result due to the 1-1 relationship between logical syntax and neural structure.  However, they cannot guarantee consistency or learn rule structure.

Differentiable inductive logic programming (ILP)~\cite{deepMinIlp2018,nttIlp21} is designed to learn logical rules using gradient descent.  This approach suffers from scalability issues as it involves generating a large number of ``rule templates'' and assigning them weights using gradient descent.  Further, this approach only learns weighted (fuzzy) rules as opposed to classical ones and also does not provide strong consistency guarantees.  In short, differentiable ILP excels at learning structure, though not in the classical sense (criterion 5) and provides a symbolic framework and explainability (criteria 1 and 2).  The work as presented does not incorporate prior knowledge or ensure consistency (criteria 3 and 4), but future work could possilby provide those extensions.  However, a key drawback of this approach is scalability (criterion 6) due to the requirement to generate large amounts of rule templates.

A key insight into this work is the combined use of annotated logic and binarized neural networks to avoid rule template generation.  The idea is that instead of generating a combinatorial number of rule templates for a given consequent, we can generate a constant number of rules and use binarized neural learning to prune the elements of the body.  Annotated logic, which has a semantics that associated atoms with elements of a lattice structure allows us to have a ``don't care'' element for a given atom in the body.  With binarized weight learning, we can then use discrete weights to assign the ``don't care'' element to weights associated with the atoms we desire to prune from the body.  We avoid the vanishing gradient problem that is associated with discrete weight learning by using mature software developed for binarized neural networks~\cite{Courbariaux16} that is the result of a line of research on the use of pseudo-gradients for binary weights during gradient descent~\cite{toms90,Magoulas1997,PlagianakosMV06}.

Neural answer set programming (NeurASP)~\cite{neurASP2020} does provide strong notions of consistency, but to date has only been used to provide a logical layer on top of a neural network to enforce logical constraints, essentially introducing a capability absent in the other approaches discussed previously.  However, this approach does not scale and does not support the learning of new rules.  In short, NeurASP satisfies criteria 1-4, but does not meet criteria 5-6 (rule structure learning and scalability).

In what follows, we review generalized annotated programs~\cite{ks92} with some key extensions (the use of a lower semi-lattice for annotations and a semantic structure that maps literals instead of atoms to annotations).  We posit that this provides the key to a framework that ensures that all criteria can be adhered to simultaneously.

\section{Generalized Annotated Programs}
\label{sec:gaps}

We now recapitulate the definition of Generalized Annotated Logic programs (from now on referred to as ``GAPs'', for short) from \cite{ks92}, but with limited syntax and certain modifications.  There are two reasons why it is advantageous to use GAPs: 
\begin{enumerate}
\item It is a framework that easily allows atomic propositions to be associated with values from a lattice structure, which generalizes other real-valued logic paradigms previously introduced in the neuro symbolic reasoning literature~\cite{lnn2020,fuzzy2022,ltn22,deepMinIlp2018}.
     
\item We can set the annotations based on a lattice structure that can support describing an atomic proposition not only as true, but as false or uncertain (i.e., no knowledge).
\end{enumerate}

\medskip
\noindent\textbf{Extension: Use of Lower Semi Lattice.} In \cite{ks92}, the authors assumed the existence of an upper semi-lattice, $\calt$ (not necessarily complete) with ordering $\sqsubseteq$.  However, in this work, we propose to instead utilize a \textit{lower} semi-lattice structure.  Therefore, we have a single element $\bot$ and multiple top elements $\top_0,\ldots\top_i\ldots\top_{max}$.  The notation $height(\calt)$ is the maximum number of elements in the lattice in a path between $\bot$ and a top element (including $\bot$ and the top element)\footnote{In general, we shall assume that the lattice consists of finite, discrete elements.}.  Note that we provide rigorous proofs of certain results from \cite{ks92} on a lower semi-lattice in Section~\ref{sec:lowerRes}.

The employment of a lower semi-lattice structure enables two desirable characteristics.  First, we desire to annotate atoms with intervals of reals in $[0,1]$ as done in previous work~\cite{lnn2020,mancalog13,apt11}.  Second, it allows for reasoning about such intervals whereby the amount of uncertainty (i.e., for interval $[l,u]$ the quantity $u-l$) decreases monotonically as an operator proceeds up the lattice structure.  
Therefore, we define the bottom element $\bot = [0,1]$ and a set of top elements $\{[x,x] \; | \; [x,x]\subseteq [0,1]\}$ (see note\footnote{N.B. that when using a semi-lattice of bounds, the notation ``$\sqsubseteq$'' loses its ``subset intuition'', as $[0,1] \sqsubseteq [1,1]$ in this case, for example.}).  Specifically, we set $\top_0 = [0,0]$ and $\top_{max}=[1,1]$.  An example of such a semi-lattice structure is shown in Figure~\ref{fig:lowerLattice}.

\begin{figure*}[!t]
    \begin{center}
        \includegraphics[width=.8\linewidth]{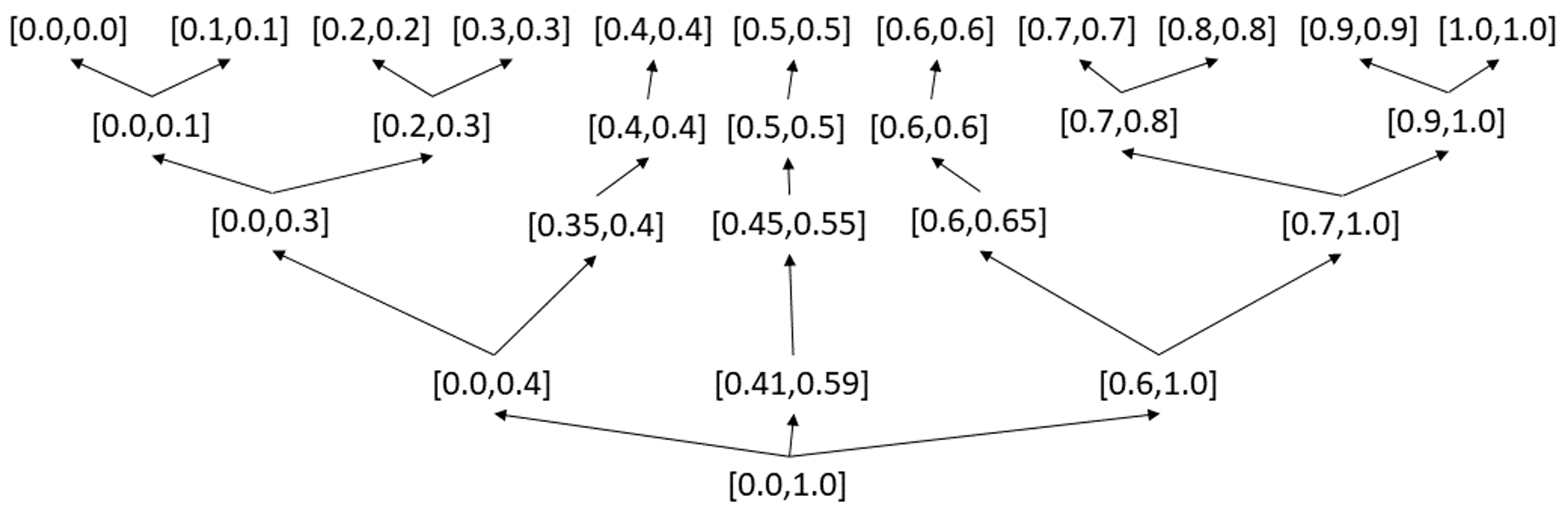}
    \end{center}
    \caption{\label{fig:lowerLattice}Example of a lower semi-lattice structure where the elements are intervals in $[0,1]$.}
\end{figure*}

\medskip
\noindent\textbf{Syntax of GAPs  (Review of prior work).}  We assume the existence of a set $\avar$ of variable symbols ranging over $\calt$ and a set $\calf$ of function symbols, each of which has an associated arity. We start by defining annotations.

\begin{definition}[Annotation]
\textsf{(i)} Any member of $\calt\,\cup \avar$ is an annotation.\\
\textsf{(ii)} If $f$ is an $n$-ary function symbol over $\calt$ and $t_1,\ldots,t_n$ are annotations, then $f(t_1,\ldots,t_n)$ is an annotation.
\end{definition}

One specific function we define is ``$\neg$'', which is used in semantics of \cite{ks92}.  For a given $[l,u]$, $\neg([l,u])=[1-u, 1-l]$.  Note that we also use the symbol  $\neg$ in our first-order language (following the formalism of \cite{ks92}).

We define a separate logical language whose constants are members of set $\calc$ and whose predicate symbols are specified by set $\calp$.
We also assume the existence of a set $\calv$ of variable symbols ranging over the constants, that no function symbols are present, and terms and atoms are defined in the usual way (cf.\ \cite{ll87}).  We shall assume that $\calc,\calp,\calv$ are discrete and finite.  In general, we shall use capital letters for variable symbols and lowercase letters for constants.  Similar to previous work~\cite{deepMinIlp2018,thomasNsr21}, we assume that all elements of $\calp$ have an arity of either~1 or~2---we use $\calp_{una}$ to denote the set of unary predicates and $\calp_{rel}$ the set of binary predicates.   We shall also denote a subsets of $\calp$ to include ``target predicates'' written $\calp_{tgt}$ that can consist of either binary or unary predicates ($\calp_{tgt\_rel},\calp_{tgt\_una}$) provided that they are not reserved words.  We shall use the symbol $\call$ to denote the set of all ground literals and $\cala$ for the set of all ground atoms.  We now define the syntactical structure of GAPs that will be used in this work.

\begin{definition}[Annotated atoms, negations, literals]
The core syntactic structures are defined as follows:
\begin{itemize}
\item{\textbf{Annotated atom.} If $\mathbf{a}$ is an atom and $\mathbf{\mu}$ is an annotation, then $\mathbf{a:\mu}$ is an \emph{annotated atom}.}
\item{\textbf{Annotated Negation.} If $\mathbf{a}$ is an atom and $\mathbf{\mu}$ is an annotation, then $\mathbf{\neg a:\mu}$ is an \emph{annotated negation}.}
\item{\textbf{Annotated Literal.} Collectively, atoms and negations are referred to as \emph{annotated literals}.}
\end{itemize}
\end{definition}

\begin{definition}[GAP Rule]
If $\ell_0:\mu_0, \ell_1:\mu_1,\ldots,\ell_m:\mu_m$ are annotated literals (such that for all $i,j \in 1,m$, $\ell_i\not\equiv \ell_j$), then
\begin{eqnarray*}
r\equiv \ell_0:\mu_0 & \leftarrow & \ell_1:\mu_1\,\wedge\,\ldots\wedge\, \ell_m:\mu_m
\end{eqnarray*}
is called a \emph{GAP rule}.  We will use the notation $head(r)$ and $body(r)$ to denote $\ell_0$ and $\{\ell_1,\ldots,\ell_m\}$ respectively. When $m=0$ ($body(r)=\emptyset$), the above GAP-rule is called a \emph{fact}.  A GAP-rule is \emph{ground} iff there are no occurrences of variables from either $\avar$ or $\calv$ in it.  For ground rule $r$ and ground literal $\ell$, $bodyAnno(\ell,r)= \mu$\textit{ such that }$\ell:\mu$\textit{ appears in the body} of $r$.  A generalized annotated program $\Pi$ is a finite set of GAP rules.
\end{definition}

\noindent\textbf{Semantics of GAPs (Extended in this work).}  The formal semantics of GAPs are defined as follows.  Note that we extend the notion of an interpretation to allow for a mapping of literals to annotations (as opposed to atoms).  However, we add a requirement on the annotation between each atom and negation that ensures equivalence to the semantic structure 
of~\cite{ks92}.  The intuition behind this extension is that we can more easily detect inconsistencies using the fixpoint operator, as we can just compare the annotations of each pair of literals (the atom and its negation).

\begin{definition}[Interpretation]
\label{def:interp}
An interpretation $I$ is any mapping from the set of all grounds literals to $\calt$ such that for literals $a, \neg a$, we have $I(a) = \neg(I(\neg a))$.  The set $\cali$ of all interpretations can be partially ordered via the ordering: $I_1\preceq I_2$ iff for all ground literals $a$, $I_1(\ell)\sqsubseteq I_2(\ell)$. $\cali$ forms a complete lattice under the $\preceq$ ordering.
\end{definition}

\noindent\textbf{Satisfaction Relationship (Review of prior work).}  We now present the concept of satisfaction:

\begin{definition}[Satisfaction]
An interpretation $I$ \emph{satisfies} a ground literal $\ell:\mu$, denoted $I\models \ell:\mu$, iff $\mu \sqsubseteq I(\ell)$. $I$ satisfies the ground GAP-rule
\begin{eqnarray*}
\ell_0: \mu_0 & \leftarrow & \ell_1:\mu_1\wedge\,\ldots\,\wedge\, \ell_m:\mu_m
\end{eqnarray*}
(denoted $I\models \ell_0:\mu_0\leftarrow \ell_1:\mu_1\,\,\wedge\,\ldots\,\wedge\,\ell_m: \mu_m$)  iff either

\begin{enumerate}
\item $I$ satisfies $\ell_0:\mu_0$ or
\item There exists an $1\leq i\leq m$ such that $I$ does not satisfy $\ell_i:\mu_i$.
\end{enumerate}

$I$ satisfies a non-ground literal or rule iff $I$ satisfies all ground instances of it.
\end{definition}

We say that an interpretation $I$ is a \emph{model} of program $\Pi$ if it satisfies all rules in $\Pi$.
Likewise, program $\Pi$ is \emph{consistent} if there exists some $I$ that is a model of $\Pi$.
We say $\Pi$ \emph{entails} $\ell:\mu$, denoted $\Pi \models_\mathit{ent} \ell:\mu$, iff for every interpretation $I$ s.t.\ $I\models \Pi$, we have that $I\models \ell:\mu$.
As shown by \cite{ks92}, we can associate a fixpoint operator with any GAP $\Pi$ that maps interpretations to interpretations.

\begin{definition}
Suppose $\Pi$ is any GAP and $I$ an interpretation.
The mapping $\T_{\Pi}$ that maps interpretations to interpretations is defined as
\[
\T_{\Pi}(I)(\ell_0) = \mathbf{sup}(annoSet_{\Pi,I}(\ell_0)),
\]
where $annoSet_{\Pi,I}(\ell_0)= \{I(\ell_0) \}\cup\{ \mu_0\: \; | \; \: \ell_0:\mu_0\leftarrow \ell_1:\mu_1\,\wedge\ldots\wedge\, \ell_m:\mu_m \textit{ is a ground instance}$ $\textit{of a rule in } \Pi,
\textit{ and for all } 1\leq  i\leq m, \textit{we have } I\models \ell_i:\mu_i\}$.
\end{definition}

The key result of \cite{ks92} (Theorem 2 in that paper, discussed in the next section) tells us that $\mathit{lfp}(\T_{\Pi})$ precisely captures the ground atomic logical consequences of $\Pi$.  We show this is also true (under the condition that $\Pi$ is consistent) even if the annotations are based on a lower lattice (see Theorem~\ref{thm:lgap_fp}).  In \cite{ks92}, the authors also define the \emph{iteration} of $\T_\Pi$ as follows:

\begin{itemize}
\item $\T_\Pi\uparrow 0$ is the interpretation that assigns $\bot$ to all ground literals.
\item $\T_\Pi\uparrow (i+1) = \T_\Pi(\T_\Pi\uparrow i)$.
\end{itemize}

For each ground $\ell \in \call$, the set $\Pi(\ell)$ is the subset of ground rules (to include facts) in $\Pi$ where $\ell$ is in the head.  We will use the notation $m_\ell$ to denote the number of rules in $\Pi(\ell)$.  
For a given ground rule, we will use the symbol $r_{\ell,i}$ to denote that it is the $i$-th rule with atom $\ell$ in the head.

\section{New Theoretical Results for Annotated Logic on a Lower Semi-lattice}
\label{sec:lowerRes}

The use of the lower semi-lattice structure for annotations in GAPs leads us to revisit some of the results of~\cite{ks92} that depend upon an upper semi-lattice structure. In this section, we shall review results from \cite{ks92} that apply to an upper semi-lattice (and shall refer to results specific to the upper semi-lattice as ``UGAPs'') and also prove analogous results for GAPs where the annotation uses a lower semi-lattice (``LGAPs'').  Further, note that in UGAPs we shall refer to interpretations defined only as mappings of atoms to annotations.  
This makes it relatively straightforward to have consistency; consider the following proposition.

\begin{proposition}
\label{prop:consist}
Any UGAP $\Pi$ consisting of UGAP-rules where the atom bodies and heads are atoms is consistent.
\end{proposition}
\begin{proof}
Consider the interpretation $I$ such that $\forall a \in \cala$, $I(a) = \top$.  By the semantics of GAPs, this will satisfy all rules.
\end{proof}

To provide a specific example of an LGAP that is not consistent, consider the following two rules (here $a$ is a ground literal):
\begin{eqnarray*}
a : [0,0] \leftarrow \\
a : [1,1] \leftarrow
\end{eqnarray*}
So, we consider the key result of \cite{ks92} below.

\begin{theorem}{(Theorem 2 in \cite{ks92})}
For UGAP $\Pi$, $\T_{\Pi}$ is monotonic and has a least fixpoint $\mathit{lfp}(\T_{\Pi})$.
Moreover, for this case $\Pi$ entails $a:\mu$ iff $\mu\sqsubseteq \mathit{lfp}(\T_{\Pi})(a)$.
\end{theorem}

However, under the condition that $\Pi$ is consistent, we can show a similar result.

\begin{theorem}
\label{thm:lgap_fp}
If LGAP $\Pi$ is consistent, then:
\begin{enumerate}
\item\label{lowerLatticeThmPt1} $\T_{\Pi}$ is monotonic,
\item\label{lowerLatticeThmPt2} $\T_{\Pi}$ has a least fixpoint $\mathit{lfp}(\T_{\Pi})$, and
\item\label{lowerLatticeThmPt3} $\Pi$ entails $a:\mu$ iff $\mu\leq \mathit{lfp}(\T_{\Pi})(a)$.
\end{enumerate}
\end{theorem}
\begin{proof}
(\ref{lowerLatticeThmPt1} and \ref{lowerLatticeThmPt2}) By creating an interpretation that maps literals to annotations, instead of atoms, the monotonicity of $\T_{\Pi}$ is trivial even in the case where $\Pi$ in inconsistent and has a least fixed point.

\noindent
(\ref{lowerLatticeThmPt3}) Suppose, BWOC that $\Pi$ entails $a:\mu$ and $\mu> \mathit{lfp}(\T_{\Pi})(a)$.  However, this would imply there is a series of logical constructs that allow us to derive $a:\mu$, and this would trivially be reflected in the iterative applications of the $\T$ operator.  Going the other way, BWOC if $\mu\leq \mathit{lfp}(\T_{\Pi})(a)$ but $\Pi$ does not entail $a:\mu$ would imply that there is no application of the constructs in $\Pi$ that lead to the deductive conclusion of $a:\mu$; however this is again contradicted by the fact that $\T$ directly leverages the elements of $\Pi$.
\end{proof}

We can also show that for both LGAPs and UGAPs, we can bound the number of applications of $\T$ until convergence.

\begin{theorem}
\label{thm:appsOfT}
If (LGAP or UGAP) $\Pi$ is consistent, then $\mathit{lfp}(\T_{\Pi}) \equiv \T_\Pi\uparrow x$ where $x = height(\calt)*|\call|$.
\end{theorem}
\begin{proof}
We know, by the definition of $\T$, for any $i \leq x$, that for all $a \in \cala$, $\T_\Pi\uparrow i(a)\sqsubseteq  \T_\Pi\uparrow x(a)$.  Hence, we just need to consider the case where $i > x$ and $\mathit{lfp}(\T_{\Pi}) \equiv \T_\Pi\uparrow i$ and $\mathit{lfp}(\T_{\Pi}) \not\equiv \T_\Pi\uparrow x$.  However, at each iteration the annotation of at least one atom must change.  The bound on the number of changes in annotation is $height(\calt)$ (as the annotations must stay the same or increase monotonically, as $\Pi$ is consistent by the statement).  Hence, we have a contradiction.
\end{proof}

We can also leverage the $\T$ operator to identify inconsistencies.

\begin{theorem}
\label{them:incon}
(LGAP or UGAP) $\Pi$ is inconsistent if and only if for value $i$, and ground atom $a$, if there exist $\mu, \mu' \in annoSet_{\Pi,\T_\Pi \uparrow i}(a)$ where $\mu \not\sqsubseteq \mu'$ and $\mu' \not\sqsubseteq \mu$.
\end{theorem}
\begin{proof}
\noindent Claim 1: If there exist $i,a$ such that the statement holds, then $\Pi$ is inconsistent.  Suppose, BWOC, that such an $i,a$ pair exist and $\Pi$ is consistent.  We know, by the definition of $\T$ that $\T(\T \uparrow i)$ must be an interpretation.  However, as there is no element above both $\mu,\mu'$, $\T$ that $\T(\T \uparrow i)$ cannot be a valid interpretation.

\noindent Claim 2: If $\Pi$ is inconsistent, then there exist $i,a$ such that the statement holds.  Suppose, BWOC, $\Pi$ is inconsistent and there does not exist such an $i,a$ pair.  Then, this implies that for all $a \in \cala$ that there exists some $i'$ where $\T(\T \uparrow i') = \T \uparrow i'$ which means for any $i'' >i'$ we have $\T \uparrow i' = \T \uparrow i''$.  Therefore, by the definition of satisfaction, $\T \uparrow i'$ must satisfy $\Pi$, which is a contradiction.
\end{proof}

\section{Parametrized GAP Rules}
\label{sec:struct}

In this section, we present a rule structure and annotation function (that later is also used as an activation function in the associated neural architecture) that can allow for us to learn GAP rules from data by using a process like gradient descent.  One of the key intuitions behind the use of a lower semi-lattice for annotations is that we can easily separate the concepts of negation and ``no information.''  Hence, in the classical case, the lattice structure would consist of three elements:
a lower ``uncertain'' element and two upper elements, one for {\em false} and one for {\em true} (see Figure~\ref{fig:lowerLattice2}).  
Through the use of a semantic structure that assigns literals to annotations, we can restrict our rules to activation functions that only modify the lower bound (i.e., if we wish to adjust the upper bound on a literal, we can instead have a rule that adjusts the lower bound on its negation).  
Also note the use of the interval of $[-1,1]$ as opposed to $[0,1]$---this is due to the common values used for binarized neural networks~\cite{Courbariaux16} that we will use to prune unneeded atoms from the body.

\begin{figure}[!t]
    \begin{center}
        \includegraphics[width=0.3\linewidth]{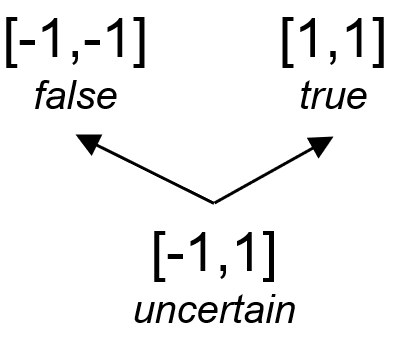}
    \end{center}
    \caption{\label{fig:lowerLattice2}Lower semi-lattice structure for the classical logic case.}
\end{figure}

\medskip
\noindent\textbf{Rule Structure.}  In this paper, we shall consider the propositional case (this is the case where all predicates are unary and we have a single constant).  Hence, there is no grounding (see Section~\ref{sec:chall} for discussions on grounding).  So, for a given literal $a$ in the rule head and for a potential set of literals $\call'$ in the body, we have the following rule $r$:
\begin{eqnarray}
\label{stdRule} r \equiv a:[f_{a,\theta_a^i}^{(i)}(X_{a}^{(i)}),1] \leftarrow \bigwedge_{\ell_j \in \call'}\ell_j : [x_j,x_j'].
\end{eqnarray}
In rule~\ref{stdRule}, we have rule $r$ that has literal $a$ in the head that is assigned a lower bound on the lattice element based on function $f_a^{(i)}$, which utilizes learned parameters $\theta_a^i$.  
Several observations:
first, note the $i$ index is used in the case where multiple rules are used with literal $a$ in the head.  
Next, $X_{a}^{(i)}$ is the vector of the lower bound of the annotations of each literal in the body ($x_j$ is the $j$th component of $X_{a}^{(i)}$).  
Finally, also note that $x_j'$ is unused by the annotation function for this particular rule (again, if we desire to use the upper bound, we can instead include the negation of literal $\ell_j$ in the set $\call'$ and use its upper bound).

\medskip
\noindent\textbf{Annotation/Activation Functions.}  In defining the annotation function $f_{a,\theta_a^i}^{(i)}$, it must have the capability to use parameters to ``turn off'' a given literal in the body.  We note that real-valued weights in traditional neural networks as well as the prior work described earlier cannot accomplish this task, hence the extensive use of fuzzy logic in other neural symbolic approaches.  
We seek to employ binarized neural networks~\cite{Courbariaux16} in which both weights and activations have values in the set $\{-1, 1\}$.  The work of~\cite{Courbariaux16} has led to successful use of a ``gradient descent'' style approach to discrete optimization for model training.  This avoids the well-known problem of vanishing gradients by substituting the partial derivative for a ``pseudo gradient'' of the activation function.  There are current implementations for binarized neural networks such as 
Larq\footnote{See \url{https://docs.larq.dev/larq/}} that have successfully employed this methodology.

For our purposes, the annotation function shall also be used as an activation function in the equivalent neural architecture.  
First, we present the meaning of the parameters.  Let $\theta_{a,j}^i$ be the $j$th component of $\theta_{a}^i$.  If this value is $1$, that means that literal $\ell_j$ should be considered in the body of the rule---likewise, if it is $-1$, then it should not.  
Hence, after training, we should be able to simply erase any $\ell_j$ where $\theta_{a,j}^i=-1$ to no effect.  
Second, we must consider the meaning of the activations.  For $f_{a,\theta_a^i}^{(i)}$ to return $1$, then for every $\theta_{a,j}^i=1$, $x_j$ must also equal $1$.  If there is any $j$ where $\theta_{a,j}^i=1$ and $x_j=-1$, then $f_{a,\theta_a^i}^{(i)}$ must return $-1$.  The following $\relu$ function accomplishes this requirement (note there are available pseudo gradients for binarized neural networks for $\relu$).

We must therefore consider the meaning of the lower bound of the annotation assigned by function $f_{a,\theta_a^i}^{(i)}$.  
In the below function, $Sign$ assigns values greater than $0$ to $1$, and $-1$ otherwise (as per~\cite{Courbariaux16}).
\begin{eqnarray}
\label{stdRelu}f_{a,\theta_a^i}^{(i)}(X_{a}^{(i)}) = Sign(\relu(1+\sum_j 0.5(1+\theta_{a,j}^i)(x_j-1)))
\end{eqnarray}

So, consider that $\theta_{a,j}^i =-1$ does not affect the sum with respect to literal $\ell_j$.  The same is true for any $x_j = 1$.  Therefore, if for all $j$ such that $\theta_{a,j}^i =1$ and $x_j=1$ then the function returns a $1$.  Now, if for a single $\ell_j$ we have $\theta_{a,j}^i =1$ and $x_j=-1$, then the function returns $-1$.

We point out that the number of rules used for a given ground atom can be considered as a hyper-parameter. Unlike in~\cite{deepMinIlp2018,nttIlp21}, where the parameters turn entire rules on or off, in this approach we are turning on or off atoms in the body.  This avoids the need to generate large numbers of rule templates.  However, it does rely on domain knowledge (e.g., in the form of a knowledge graph) to limit the number of literals considered in a given rule (in other words, we want to keep set $\call'$ small).

\section{Neural Architecture for Learning GAPs}
\label{sec:neuralEmbed}

In this section, we look at how a GAP can be embedded in a neural framework for use in training by gradient descent.  The concept is similar to that of the differentiable inductive logic programming (ILP) literature~\cite{nttIlp21,deepMinIlp2018} in that it involves an unrolling of the fixed point operator.  However, we use the rule structure and annotation/activation function of the previous section to avoid generating numerous rule templates.

We shall assume a RNN neural architecture consisting of $K$ recurrent cells.  In order to assure correctness, $K$ must be set to the maximum number of applications of the fixpoint operator (see Theorem~\ref{thm:appsOfT}).  We use $A$ to denote a vector of elements in $\calf$ of size $n$---intuitively, we want each position of $A$ to correspond to the annotation of a single literal in $\call$.  So, we shall assume a numbering of literals such that literal $a_j\in\cala$ would correspond with the $j$-th position of $A$.  The initial input to the first recurrent cell will be denoted $A_0$, where all positions are set to $\bot$.  For recurrent cell $t$, $A_{t-1}$ is the input and $A_{t}$ is the output.

For each rule $r_{a_j,i}$ there is an associated vector of body annotation lower bounds $X_{a_j}^{(i)}$ of size ${n_{a_j,i}}\leq n$; this is of a smaller size than $A_t$ as a body of a given rule may not include all atoms.  However, for a given annotation function $f_{a}^{(i)}$ in the head of a rule, when we say it evaluates an $n$-sized vector $A_t$, it is actually evaluating the positions in $A_t$ corresponding to the atoms for which there are positions in $X_{a_j}^{(i)}$.

In each recurrent cell $t$, vector $\tilde{A}_t^{(j)}$ is created and is of length $m_{a_j}$ (the number of rules, including facts, with atom $a_j$ in the head).  Position $i$ of $\tilde{A}_t^{(j)}$ is equal to $f_{a_j}^{(i)}(\theta_{a_j}^{(i)},A_{t-1})$.  For facts, this position will simply be the annotation in the head of the fact.

Finally, the result of recurrent cell $t$ is vector $A_t$ in which the $j$-th position of this vector corresponds with the supremum of the corresponding $\tilde{A}_t^{(j)}$.  Note that as this is a supremum of annotations in which only the lower bound changes, this is equivalent to a max pooling layer that only applies the pooling function to a subset of neurons.  The intuition is that each cell corresponds with one application of the $\T$ operator.  Figure~\ref{rnnCells} shows a depiction of this architecture.

\begin{figure*}[t]
    \begin{center}
        \includegraphics[width=.9\linewidth]{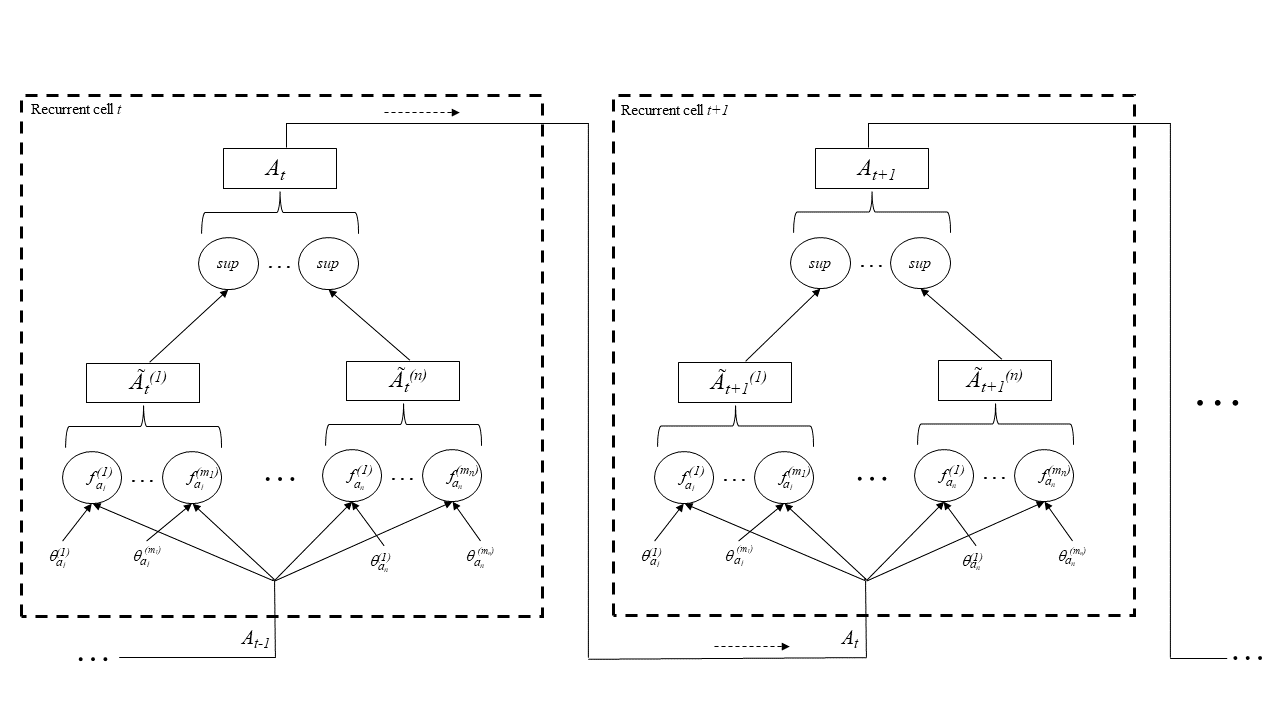}
    \end{center}
    \caption{\label{rnnCells}Recurrent cells described in this section.}
\end{figure*}

\section{Identifying Inconsistencies}
\label{sec:incon}
The use of the lower lattice for annotations does not allow for easy guarantees of consistency (i.e., see Proposition~\ref{prop:consist}).  While it is possible to guarantee consistency by determining rules with restrictive structures (e.g., disallowing negation in learned rules as in \cite{deepMinIlp2018}), this becomes difficult if rules are created as a result of the training process (e.g., setting parameters that configure the rules through gradient descent).

Other approaches to dealing with consistency in the settings integrating symbolic and machine learning paradigms have included notions of quantifying a level of inconsistency and then attempting to minimize such inconsistency.  This is commonly used in the (non-neural) probabilistic soft logic (PSL)~\cite{getoor2017}.  This concept has also appeared in neuro symbolic approaches as well~\cite{lnn2020}.

However, an advantage with the approach of this paper is that an inconsistency will be detected by some application of the fixpoint operator (Theorem~\ref{them:incon}).  This is why we changed the domain of an interpretation to be literals instead of atoms, as we can more easily identify a case where an atom and its negation are assigned annotations in such a way to cause an inconsistency.  We create a special atomic symbol $\incon$ for this purpose. The negation of this symbol is never used; however, if this symbol is annotated with $[1,1]$ (``true'') then it means we have an inconsistency.  In the logic program, we add a single rule of the form~\ref{rule:incon1} below, and for each ground atom $a$ we add an instance of rule~\ref{rule:incon2}:
\begin{eqnarray}
\label{rule:incon1}\incon:[-1,1] \leftarrow \\
\label{rule:incon2}\incon:[Sign(x+y),1] \leftarrow a:[x,x'] \wedge \neg a:[y,y']
\end{eqnarray}
Hence, for any application of the fixpoint operator, we know there is an inconsistency if $\incon$ is true and we need not carry out further applications of the operator.  
Our intuition is that this can be used during the training process to either establish a loss function that disallows any inconsistency, or to conduct a check at each iteration of gradient descent.  It is noteworthy that the fixpoint operator, which also is likely to bring further efficiencies to the training process, can be used in the forward pass as opposed to the equivalent neural architecture.  Further, the fixpoint operator will also provide information on {\em where} the inconsistency occurs, which can also be useful in guiding the training process.  
Note that some domains may accept inconsistency if it is localized and not associated with certain hard constraints (e.g., specified by rules with no parameters, which this framework also supports).

\section{How the Framework Meets the Criteria}
\label{sec:discusion}

We now briefly review the criteria introduced in Section~\ref{sec:cri} and discuss how the proposed approach can meet them:
\begin{itemize}
\item As it is an inherently neuro symbolic approach, it meets criterion 1.  

\item Due to the direct mapping to an equivalent neural architecture, which is similar to LNN~\cite{lnn2020}, it meets criterion 2 in that the results are directly explainable, symbolic, and amenable to further analysis.  
    
\item The framework can both allow parametrized rules (for induction) as well as directly represent logical statements in the equivalent neural architecture, 
which satisfies criterion 3---incorporation of prior knowledge---(as does \cite{ltn22,thomasNsr21,lnn2020}). 

\item We can provably show it can guarantee consistency, even when using negation, meeting criterion 4 (as does \cite{neurASP2020}).  

\item Through the use of parametrized rules, annotated logic, and binarized neural networks, it enables rule learning (criteria 5).

\item Unlike~\cite{deepMinIlp2018}, it avoids (non-scalable) template generation, instead building up antecedents by using the lattice structure to ``turn off'' certain atoms (meeting criterion 6).
\end{itemize}
Though, as we have argued, our proposed formalism meets all six criteria, there remain several hurdles to overcome, which we discuss in the following section.

\section{Challenges and Conclusion}
\label{sec:chall}
While the use of annotated logic in a neural symbolic framework is promising, there are several challenges to be addressed.  
First, it is expected that any neuro symbolic approach should support first order logic (i.e., the non-ground case).  
While there is nothing presented here that would not support such a capability, practical considerations around grounding can limit scalability in practice.  
We are exploring the use of knowledge graphs~\cite{tplp13} as an ontology to limit relationships among constants to reduce the grounding problem.  
A second, but equally important, concern is dealing with inconsistency.  It is challenging to specifically avoid inconsistency and ensure that the training process continues.  
Ultimately, this will require a determination of best practices around gradient descent and a vigilant measure of the boundaries of inconsistency.  
Third, and in the same vain, we also will need to explore other issues such as establishing best practices for connecting a neural network associated with a logic program with lower-level neural networks used for perception (e.g., CNN's).  Finally, implementation issues will clearly be very important in successfully developing tools based on this framework.

\section*{Acknowledgment}
P.S.\ is supported by internal funding from the ASU Fulton Schools of Engineering.
G.S.\ is supported by
Universidad Nacional del Sur (UNS) under grant PGI~24/ZN34
and Agencia Nacional de Promoci\'on Cien\-t\'{\i}\-fi\-ca y Tecnol\'ogica
under grant PICT-2018-0475 (PRH-2014-0007).



\end{document}